\documentclass[reprint,amsmath,amssymb,aps,pra,groupedaddress,floatfix]{revtex4-2}

\usepackage{graphicx}
\usepackage[final]{microtype}
\usepackage{amsthm}
\usepackage{algpseudocode}
\usepackage{booktabs}
\usepackage{multirow}
\usepackage{xcolor}
\usepackage{hyperref}
\usepackage{tikz}
\usepackage{subcaption}
\usepackage[newcommands]{ragged2e}
\usepackage[labelfont=bf,labelsep=colon,justification=Justified,singlelinecheck=false,font=small,format=plain,compatibility=false]{caption}
\usetikzlibrary{shapes,arrows,positioning,calc,decorations.pathmorphing,shapes.geometric}

\newtheorem{theorem}{Theorem}
\newtheorem{proposition}{Proposition}

\newcommand{\R}{\mathbb{R}}

\newcommand{\E}{\mathbb{E}}
\newcommand{\vect}[1]{\mathbf{#1}}
\newcommand{\mat}[1]{\mathbf{#1}}

\begin{document}
\captionsetup{justification=Justified,singlelinecheck=false}

\title{Selective Synchronization Attention}

\author{Hasi Hays}
\email{hasih@uark.edu}
\affiliation{Department of Chemical Engineering, University of Arkansas, Fayetteville, AR 72701, USA}

\date{\today}

\begin{abstract}
The Transformer architecture has become the foundation of modern deep learning, yet its core self-attention mechanism suffers from quadratic computational complexity and lacks grounding in biological neural computation. We propose \emph{Selective Synchronization Attention} (SSA), a novel attention mechanism that replaces the standard dot-product self-attention with a closed-form operator derived from the steady-state solution of the Kuramoto model of coupled oscillators. In SSA, each token is represented as an oscillator characterized by a learnable natural frequency and phase; the synchronization strength between token pairs, determined by a frequency-dependent coupling and phase-locking condition, serves as the attention weight. This formulation provides three key advantages: (i) natural sparsity arising from the phase-locking threshold, whereby tokens with incompatible frequencies automatically receive zero attention weight without explicit masking; (ii) unified positional-semantic encoding through the natural frequency spectrum, eliminating the need for separate positional encodings; and (iii) a single-pass, closed-form computation that avoids iterative ODE integration, with all components (coupling, order parameter, synchronization) derived from the oscillatory framework. We instantiate SSA within the Oscillatory Synchronization Network (OSN), a drop-in replacement for the Transformer block. Analysis of the synchronization matrices reveals non-uniform, head-diverse coupling patterns even at initialization, demonstrating a stronger architectural inductive bias than the approximately uniform attention produced by randomly initialized Transformers.
\end{abstract}

\maketitle

%% ============================================================
%%  1. INTRODUCTION
%% ============================================================
\section{Introduction}\label{sec:intro}

The Transformer architecture~\cite{vaswani2017attention,hays2026encyclopedia} has fundamentally reshaped deep learning, achieving state-of-the-art performance across natural language processing, computer vision, and scientific computing~\cite{hays2025hierarchical}. At the heart of the Transformer lies the self-attention mechanism, which computes pairwise interactions between all tokens in a sequence:
\begin{equation}\label{eq:std_attn}
    \text{Attn}(\mat{Q}, \mat{K}, \mat{V}) = \text{softmax}\!\left(\frac{\mat{Q}\mat{K}^\top}{\sqrt{d_k}}\right) \mat{V},
\end{equation}
where $\mat{Q}, \mat{K}, \mat{V} \in \R^{N \times d_k}$ are the query, key, and value projections of an input sequence of length $N$. While powerful, this mechanism has two fundamental limitations. First, its $O(N^2 d_k)$ computational complexity and $O(N^2)$ memory footprint create a prohibitive bottleneck for long sequences. Second, despite its empirical success, dot-product attention has no clear analogue in biological neural computation.

In contrast, the brain routes information through \emph{oscillatory synchronization}. The Communication Through Coherence (CTC) hypothesis~\cite{fries2005mechanism,fries2015rhythms} proposes that effective neuronal communication requires phase alignment of oscillatory activity between sending and receiving populations: coherent oscillations open temporal windows for synaptic transmission, while desynchronized populations are effectively disconnected. This principle, binding by synchrony~\cite{singer1995visual,engel2001dynamic}, provides a biologically grounded mechanism for selective information routing that operates without computing explicit pairwise similarity scores.

Recent work has begun bridging oscillatory dynamics and deep learning. The AKOrN framework~\cite{rende2024mapping} introduced Kuramoto-type dynamics at the neuron level within neural network layers, demonstrating that oscillatory interactions can enhance representation learning. However, AKOrN modifies neuronal activations rather than the attention mechanism itself, and relies on iterative ODE integration during the forward pass, which limits scalability. Separately, efficient attention alternatives such as Mamba~\cite{gu2023mamba}, RetNet~\cite{sun2023retentive}, and linear attention~\cite{katharopoulos2020transformers} address the quadratic complexity of attention but do not draw on the oscillatory synchronization principles that govern biological neural computation. A fundamental question thus emerges: can the principles of oscillatory synchronization, which the brain employs to orchestrate information flow across billions of neurons, be\linebreak harnessed to replace the attention mechanism itself, rather than merely augmenting learned representations? An ideal solution would combine the computational advantages of efficient attention with the biological grounding of oscillatory dynamics, providing both scalability and interpretability within a single coherent framework.

In this work, we propose \emph{Selective Synchronization Attention} (SSA), a mechanism that \emph{replaces} the dot-product self-attention with a closed-form operator derived from the steady-state Kuramoto model~\cite{kuramoto1984chemical,strogatz2000kuramoto}. Each token is represented as an oscillator with a learnable natural frequency $\omega_i$ and initial phase $\theta_i$; the synchronization strength between tokens $i$ and $j$, determined by their frequency compatibility and coupling, serves as the attention weight. The key insight is that the Kuramoto model admits analytical steady-state solutions for the pairwise synchronization landscape, enabling us to compute attention weights in a single forward pass without iterative ODE integration. This closed-form solution transforms the Kuramoto model from a dynamical system requiring numerical simulation into a feed-forward operator suitable for gradient-based training.

This work makes four principal contributions toward bridging oscillatory neural dynamics and modern sequence modeling architectures. First, we derive SSA, a closed-form attention operator grounded in the steady-state solution of the Kuramoto model of coupled oscillators. By exploiting the analytical phase-locking conditions of the Kuramoto mean field, SSA computes synchronization-based attention weights in a single forward pass without requiring iterative ODE integration, providing a principled and computationally efficient replacement for dot-product self-attention. Second, we show that SSA provides natural sparsity through the phase-locking condition: token pairs whose natural frequencies are incompatible, that is, whose frequency mismatch exceeds the product of the global coupling strength, the emergent order parameter, and the frequency-dependent coupling, automatically receive zero synchronization weight. This intrinsic sparsity arises from the physics of the model rather than from external heuristics such as top-$k$ selection or explicit masking, and the degree of sparsity is controllable through the coupling parameters. Crucially, the coupling itself is derived from the oscillatory framework (oscillators with similar frequencies couple more strongly), the order parameter is computed empirically from the phase distribution at each forward pass, and the coupling strength uses a smooth positive parameterization, ensuring that every component of the SSA operator maintains physical grounding.

Third, we demonstrate that the learnable natural frequency spectrum serves as a unified positional-semantic encoding. Because the synchronization between two tokens depends on the difference between their frequency vectors, and because gradient descent on a language modeling objective drives nearby tokens toward compatible frequencies (Proposition~\ref{thm:pe_emerges}), the frequency representation naturally encodes both the sequential position and the semantic content of each token, eliminating the need for separate positional embeddings. Fourth, we instantiate SSA within the Oscillatory Synchronization Network (OSN), a drop-in replacement for the standard Transformer block that preserves the input-output interface while replacing multi-head attention with multi-frequency synchronization heads. GPU benchmarks on an NVIDIA A100 demonstrate near-identical parameter counts with standard Transformer blocks, structured synchronization patterns, and emergent phase coherence, validating the mechanism as a functional drop-in replacement for dot-product attention.

%% ============================================================
%%  2. RELATED WORK
%% ============================================================
\section{Related Work}\label{sec:related}

\subsection{Efficient Attention Mechanisms}\label{sec:efficient_attn}

The quadratic complexity of standard attention~\cite{hays2026attention} has motivated numerous efficiency improvements. FlashAttention~\cite{dao2022flashattention} reduces memory overhead through IO-aware tiling but retains $O(N^2)$ computation. Sparse attention patterns~\cite{child2019generating} restrict the attention field to local windows and fixed strides. The Performer~\cite{choromanski2021rethinking} approximates the softmax kernel via random features, achieving linear complexity but with approximation error. Linear Transformers~\cite{katharopoulos2020transformers} replace the softmax with a kernel decomposition that enables $O(N)$ computation through the associativity of matrix products. These methods preserve the dot-product attention paradigm while reducing its cost; SSA instead replaces the paradigm entirely with an oscillatory synchronization mechanism.

\subsection{Attention Alternatives}\label{sec:attn_alternatives}

Several architectures have proposed complete alternatives to attention. Structured state space models (S4)~\cite{gu2022efficiently} and their selective variant Mamba~\cite{gu2023mamba} process sequences through linear recurrences, achieving $O(N)$ complexity with strong performance on long-range tasks. RetNet~\cite{sun2023retentive} combines retention mechanisms with linear recurrence for efficient inference. RWKV~\cite{peng2023rwkv} reformulates attention as a linear RNN. Hyena~\cite{poli2023hyena} replaces attention with long convolutions parameterized by implicit neural networks. While these methods address efficiency, they depart from the attention paradigm without providing the interpretable, biologically motivated mechanism that SSA offers.

\subsection{Oscillatory Models in Machine Learning}\label{sec:osc_models}

Oscillatory dynamics have appeared in machine learning primarily through the Kuramoto model~\cite{kuramoto1984chemical}, with resonance-based architectures also emerging in geometric deep learning~\cite{hays2026rsgn}. KuraNet~\cite{ricci2021kuranet} learned to predict Kuramoto model dynamics from visual inputs. The Kuramoto model has been applied in graph neural networks for community detection and graph partitioning~\cite{rodrigues2016kuramoto,hays2026RAG-GNN}. Most relevant to our work, AKOrN~\cite{rende2024mapping} introduced Kuramoto oscillator dynamics at the neuron level, replacing standard activations with oscillatory states that are iteratively updated via ODE integration. Critically, AKOrN does \emph{not} replace the attention mechanism: it modifies the representation within layers while still relying on standard attention for inter-token communication. In contrast, SSA replaces the attention mechanism itself with a closed-form synchronization operator, operating at the token level rather than the neuron level, and avoiding iterative ODE solvers entirely.

\subsection{Neuroscience of Oscillatory Attention}\label{sec:neuro}

The Communication Through Coherence (CTC) framework~\cite{fries2005mechanism,fries2015rhythms} provides the theoretical foundation for our approach. CTC proposes that neuronal communication is gated by the coherence of oscillatory activity: populations that oscillate in phase can communicate effectively, while desynchronized populations are functionally disconnected. Singer and Gray~\cite{singer1995visual} demonstrated that temporal synchronization in the gamma band ($30$--$80$~Hz) binds distributed feature representations into coherent percepts. Engel et al.~\cite{engel2001dynamic} extended this to a framework of dynamic predictions where synchronization patterns implement flexible, context-dependent information routing. Cross-frequency coupling, particularly theta-gamma phase-amplitude coupling~\cite{canolty2010oscillatory}, further organizes information across temporal scales. SSA translates these principles into a differentiable computational mechanism.

%% ============================================================
%%  3. METHOD
%% ============================================================
\section{Method}\label{sec:method}

\subsection{Preliminaries: The Kuramoto Model}\label{sec:kuramoto}

The Kuramoto model~\cite{kuramoto1984chemical} describes the dynamics of $N$ coupled oscillators, each characterized by a phase $\theta_i \in [0, 2\pi)$ and a natural frequency $\omega_i$:
\begin{equation}\label{eq:kuramoto}
    \frac{d\theta_i}{dt} = \omega_i + \frac{K}{N} \sum_{j=1}^{N} \sin(\theta_j - \theta_i), \quad i = 1, \ldots, N,
\end{equation}
where $K > 0$ is the global coupling strength. The collective behavior is captured by the \emph{order parameter}:
\begin{equation}\label{eq:order_param}
    r \, e^{i\psi} = \frac{1}{N} \sum_{j=1}^{N} e^{i\theta_j},
\end{equation}
where $r \in [0, 1]$ measures the degree of global phase coherence and $\psi$ is the mean phase. When $r \approx 1$, the oscillators are nearly synchronized; when $r \approx 0$, they are incoherent.

The classical result of Kuramoto~\cite{kuramoto1984chemical} and Strogatz~\cite{strogatz2000kuramoto} shows that for a population with natural frequencies drawn from a symmetric unimodal distribution $g(\omega)$, there exists a critical coupling strength $K_c = 2/(\pi g(0))$ above which a macroscopic fraction of oscillators locks to a common frequency. At steady state, an oscillator $i$ is \emph{phase-locked} to the mean field if and only if $|\omega_i - \psi| \leq Kr$, where $\psi$ is the frequency of the mean field. This phase-locking condition is the key insight that we exploit for attention.

\subsection{Selective Synchronization Attention}\label{sec:ssa}

We now derive the SSA mechanism. Given an input sequence $\mat{X} = [\vect{x}_1, \ldots, \vect{x}_N]^\top \in \R^{N \times D}$, we compute the following oscillatory representations.

\paragraph{Phase and frequency initialization.}
Each token $i$ is assigned a natural frequency vector and an initial phase vector via learned projections:
\begin{align}
    \boldsymbol{\omega}_i &= \mat{W}_\omega \, \vect{x}_i \in \R^d, \label{eq:freq_proj} \\
    \boldsymbol{\theta}_i &= \mat{W}_\theta \, \vect{x}_i \in \R^d, \label{eq:phase_proj}
\end{align}
where $\mat{W}_\omega, \mat{W}_\theta \in \R^{d \times D}$ are learnable weight matrices and $d = D/H$ is the head dimension for $H$ heads.

\paragraph{Frequency-dependent coupling.}
The pairwise coupling strength between oscillators $i$ and $j$ is derived from the oscillatory framework through the principle of frequency clustering: oscillators with similar natural frequencies couple more strongly, a well-established phenomenon in coupled oscillator theory~\cite{strogatz2000kuramoto,acebron2005kuramoto}:
\begin{equation}\label{eq:coupling}
    J_{ij}^{(h)} = \exp\!\left(-\alpha_h \, \|\boldsymbol{\omega}_i - \boldsymbol{\omega}_j\|_2^2\right),
\end{equation}
where $\alpha_h > 0$ is a learnable coupling bandwidth parameter for each synchronization head $h$, parameterized as $\alpha_h = \text{softplus}(\hat{\alpha}_h)$ to ensure positivity. This yields $J_{ij} \in (0, 1]$, with $J_{ij} \approx 1$ for oscillators with similar frequencies and $J_{ij} \approx 0$ for dissimilar ones. Crucially, the coupling is entirely determined by the frequency representations, maintaining full consistency with the oscillatory framework: the coupling strength is a structural property of the oscillator network that depends on frequency compatibility rather than on an auxiliary learned scoring function.

\paragraph{Closed-form phase-alignment operator.}
The core of SSA is the synchronization matrix $\mat{S} \in \R^{N \times N}$, derived from the steady-state Kuramoto analysis. Consider the pairwise synchronization between oscillators $i$ and $j$ in the mean-field approximation. Their relative phase $\phi_{ij} = \theta_i - \theta_j$ satisfies at steady state~\cite{strogatz2000kuramoto,acebron2005kuramoto}:
\begin{equation}\label{eq:steady_state}
    \sin(\phi_{ij}^*) = \frac{\Delta\omega_{ij}}{K \, r \, J_{ij}},
\end{equation}
where $\Delta\omega_{ij} = \|\boldsymbol{\omega}_i - \boldsymbol{\omega}_j\|_2$ is the frequency mismatch and $K$ is a learnable global coupling strength. This equation has a solution if and only if the \emph{phase-locking condition} is satisfied:
\begin{equation}\label{eq:phase_lock}
    |\Delta\omega_{ij}| \leq K \, r \, J_{ij}.
\end{equation}

When the condition holds, the synchronization strength, measured by the coherence $\cos(\phi_{ij}^*)$, is:
\begin{equation}\label{eq:phase_align}
    S_{ij} = J_{ij} \cdot \cos\!\left(\arcsin\!\left(\frac{\Delta\omega_{ij}}{K \, r \, J_{ij}}\right)\right).
\end{equation}
Using the identity $\cos(\arcsin(x)) = \sqrt{1 - x^2}$, this simplifies to:
\begin{equation}\label{eq:phase_align_sqrt}
    S_{ij} = \begin{cases}
        J_{ij} \sqrt{1 - \left(\dfrac{\Delta\omega_{ij}}{K \, r \, J_{ij}}\right)^2} & \text{if } \Delta\omega_{ij} \leq K \, r \, J_{ij}, \\[8pt]
        0 & \text{otherwise}.
    \end{cases}
\end{equation}

The synchronization matrix $\mat{S}$ is the direct analogue of the attention weight matrix in standard Transformers: $S_{ij}$ measures how strongly token $j$ should contribute to the representation of token $i$. The crucial difference is that $\mat{S}$ is \emph{naturally sparse}; pairs with large frequency mismatch are automatically excluded by the phase-locking condition, without any explicit top-$k$ selection or masking. Figure~\ref{fig:oscillators} illustrates this mechanism: tokens with compatible natural frequencies form synchronized clusters on the phase circle (\ref{fig:phase_space}), occupying distinct regions in frequency space (\ref{fig:freq_space}), which produces the block-diagonal synchronization matrix (\ref{fig:sync_matrix}) where within-cluster pairs synchronize and cross-cluster pairs receive zero attention weight.

\begin{figure*}[t]
\centering
\includegraphics[width=\textwidth]{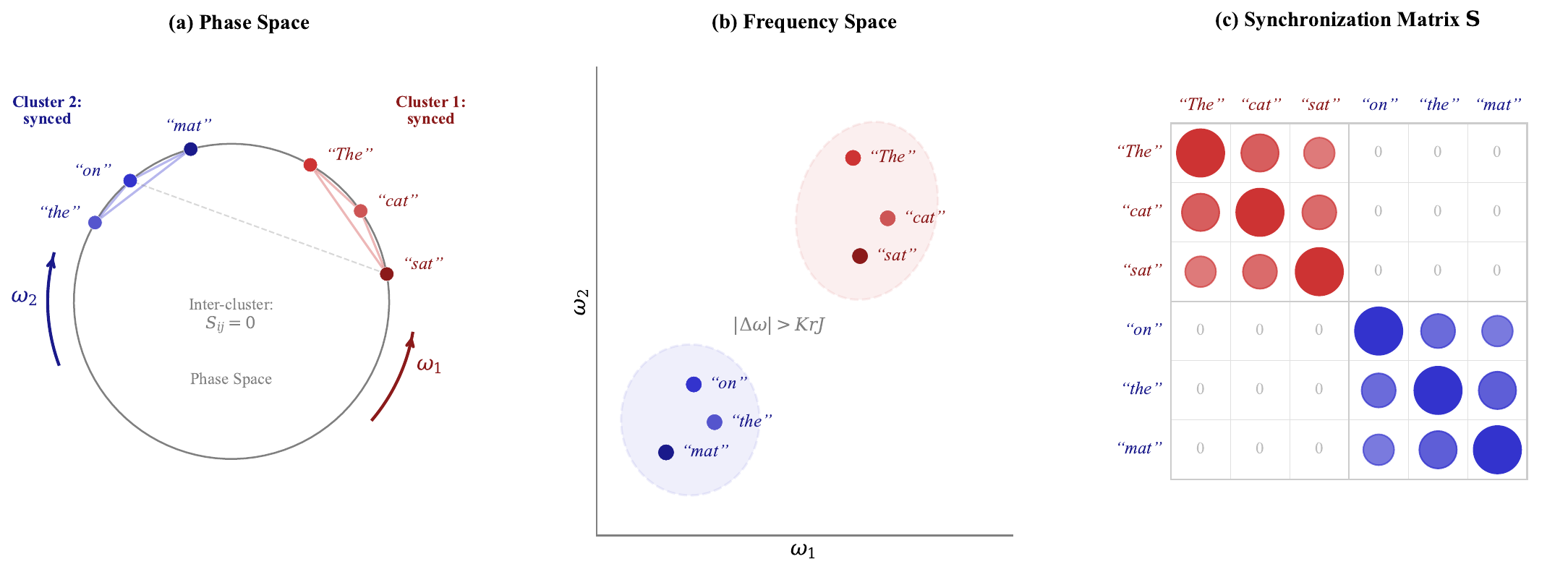}%
\parbox{0pt}{\phantomsubcaption\label{fig:phase_space}}%
\parbox{0pt}{\phantomsubcaption\label{fig:freq_space}}%
\parbox{0pt}{\phantomsubcaption\label{fig:sync_matrix}}%
\caption{Selective Synchronization Attention: conceptual illustration for the sentence ``The cat sat on the mat.'' \textbf{(a)}~Phase space: tokens are represented as oscillators on the phase circle. Tokens with similar natural frequencies ($\omega_1$ or $\omega_2$) synchronize into coherent clusters (red and blue), enabling information exchange via solid connections. Adjacent tokens from different clusters (``sat'' and ``on'') remain desynchronized ($S_{ij} = 0$, dashed line), producing natural sparsity. \textbf{(b)}~Frequency space: the two clusters occupy distinct regions, separated by the phase-locking threshold $|\Delta\omega| > KrJ$. \textbf{(c)}~Synchronization matrix $\mathbf{S}$: the resulting block-diagonal structure, with strong within-cluster synchronization and zero cross-cluster weights.}
\label{fig:oscillators}
\end{figure*}

In practice, we compute the order parameter $r$ directly from the phase distribution at each forward pass, consistent with the Kuramoto formulation (Eq.~\ref{eq:order_param}):
\begin{equation}\label{eq:empirical_r}
    r = \frac{1}{d} \sum_{l=1}^{d} \left|\frac{1}{N} \sum_{j=1}^{N} e^{i\theta_j^{(l)}}\right|,
\end{equation}
where the average is taken over the $d$ phase dimensions. This makes $r$ an emergent property of the current token population's phase coherence rather than a free parameter, preserving the physical meaning of the order parameter. The global coupling strength is parameterized as $K = \text{softplus}(\hat{K})$ for a learnable scalar $\hat{K}$, ensuring $K > 0$ with smooth gradients. With clamping for numerical stability:
\begin{equation}\label{eq:practical_S}
    S_{ij} = J_{ij} \cdot \cos\!\left(\arcsin\!\left(\text{clamp}\!\left(\frac{\Delta\omega_{ij}}{K \, r \, J_{ij} + \epsilon}, -1, 1\right)\right)\right),
\end{equation}
with $S_{ij}$ set to zero where $\Delta\omega_{ij} > K \, r \, J_{ij}$.

\paragraph{Output computation.}
The value vectors $\mat{V} = \mat{X} \mat{W}_V$ are aggregated using the normalized synchronization weights:
\begin{equation}\label{eq:output}
    \vect{y}_i = \frac{\sum_{j=1}^{N} S_{ij} \, \vect{v}_j}{\sum_{j=1}^{N} S_{ij} + \epsilon},
\end{equation}
followed by a linear output projection $\mat{W}_O$.

\subsection{Multi-Frequency Synchronization Heads}\label{sec:mfsh}

Analogous to multi-head attention, we employ $H$ \emph{synchronization heads}, each operating on a $d$-dimensional subspace of the frequency and phase representations. Each head $h$ computes its own synchronization matrix $\mat{S}^{(h)}$ using independent frequency projections $\mat{W}_\omega^{(h)}$, phase projections $\mat{W}_\theta^{(h)}$, and coupling parameters. The outputs are concatenated and projected:
\begin{equation}\label{eq:mfsh}
    \text{MFSH}(\mat{X}) = \text{Concat}(\vect{y}^{(1)}, \ldots, \vect{y}^{(H)}) \, \mat{W}_O.
\end{equation}

Different heads can learn to operate at different frequency scales, analogous to the multi-band structure of neural oscillations where theta, alpha, beta, and gamma bands serve different computational roles~\cite{canolty2010oscillatory}.

\subsection{The OSN Block}\label{sec:osn_block}

The full \emph{Oscillatory Synchronization Network} (OSN) block follows the pre-norm Transformer architecture with SSA replacing multi-head attention:
\begin{align}
    \vect{z} &= \mat{X} + \text{Dropout}(\text{MFSH}(\text{LayerNorm}(\mat{X}))), \label{eq:osn_attn} \\
    \mat{Y} &= \vect{z} + \text{Dropout}(\text{FFN}(\text{LayerNorm}(\vect{z}))), \label{eq:osn_ffn}
\end{align}
where FFN is a two-layer feed-forward network with GELU activation~\cite{ba2016layer}. The OSN block has the same input-output interface ($\R^{N \times D} \to \R^{N \times D}$) as a standard Transformer block, enabling drop-in replacement. Figure~\ref{fig:architecture} compares the architecture of the standard Transformer block with the OSN block, highlighting the structural correspondence between multi-head attention and multi-frequency synchronization heads.

\begin{figure*}[t]
\centering
\includegraphics[width=\textwidth]{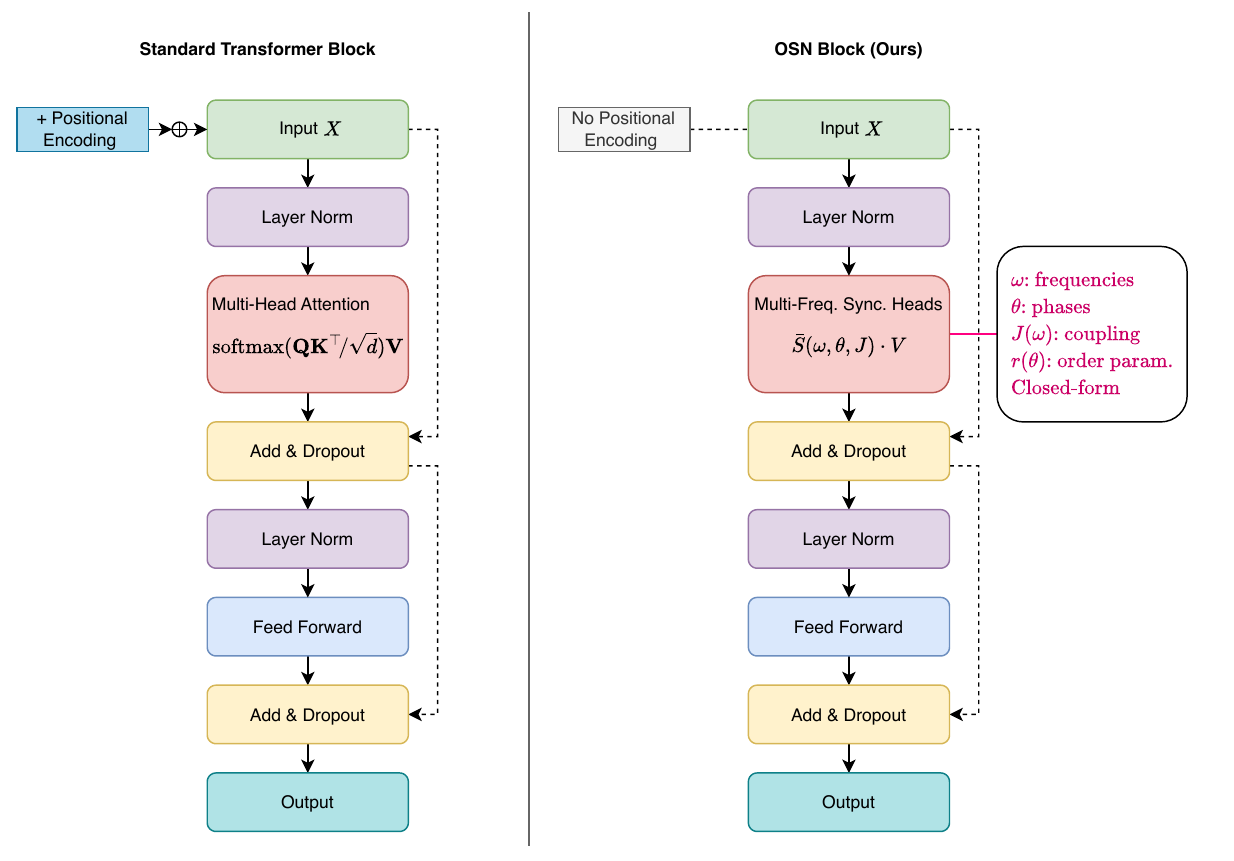}
\caption{Architecture comparison. \textbf{Left:} Standard Transformer block with multi-head dot-product attention and explicit positional encoding. \textbf{Right:} OSN block with Multi-Frequency Synchronization Heads using the closed-form phase-alignment operator. The blocks share identical input-output interfaces ($\R^{N \times D} \to \R^{N \times D}$), enabling drop-in replacement.}
\label{fig:architecture}
\end{figure*}

\subsection{Complexity Analysis}\label{sec:complexity}

We analyze the computational complexity of SSA and compare it with standard attention.

\begin{proposition}[Complexity of SSA]\label{prop:complexity}
The computational complexity of SSA is:
\begin{enumerate}
    \item \emph{Dense computation:} $O(N^2 d)$, matching standard attention. The frequency-dependent coupling $J_{ij} = \exp(-\alpha\|\boldsymbol{\omega}_i - \boldsymbol{\omega}_j\|^2)$ is computed directly from the pairwise frequency differences at no additional asymptotic cost, and the empirical order parameter $r$ requires $O(NHd)$.
    \item \emph{Sparse coupling with top-$k$:} $O(N k d)$ where $k \ll N$, achieving linear scaling in sequence length.
    \item \emph{Phase-locking sparsity:} When natural frequencies are spread over a range $2\Omega$ with coupling threshold $KrJ$, the expected number of non-zero entries per row is $O(N \cdot KrJ / \Omega)$, yielding effective complexity $O(N^2 \cdot KrJ / \Omega \cdot d)$.
\end{enumerate}
\end{proposition}

The frequency projection and phase projection each require $O(NDd)$ computation. The pairwise frequency difference $\Delta\omega_{ij}$ is $O(N^2 d)$ and simultaneously provides both the coupling matrix $J$ (via the Gaussian kernel) and the synchronization condition, eliminating the need for a separate coupling computation. When combined with top-$k$ sparsification or when the phase-locking condition provides sufficient natural sparsity, the overall complexity reduces to $O(Nkd)$ for effective neighborhood size $k$.

\subsection{Theoretical Properties}\label{sec:theory}

We state three theoretical results characterizing the expressiveness and properties of SSA: two theorems establishing universality and natural sparsity, and a proposition characterizing the emergent positional encoding bias. Proofs are provided in Appendix~\ref{app:proofs}.

\begin{theorem}[Universality]\label{thm:universality}
For any continuous function $f: \mathcal{X} \to \R^{N \times D}$ on a compact domain $\mathcal{X} \subset \R^{N \times D}$ and any $\epsilon > 0$, there exists an OSN with sufficient depth $L$ and width $D$ such that $\|\text{OSN}(\mat{X}) - f(\mat{X})\| < \epsilon$ for all $\mat{X} \in \mathcal{X}$.
\end{theorem}

\begin{theorem}[Natural Sparsity]\label{thm:sparsity}
Let the natural frequencies $\omega_1, \ldots, \omega_N$ be drawn i.i.d.\ from a distribution $g$ with support $[-\Omega, \Omega]$. With uniform coupling $J_{ij} = J$, the expected fraction of non-zero entries in $\mat{S}$ is:
\begin{multline}
    \E\!\left[\frac{|\{(i,j) : S_{ij} > 0\}|}{N^2}\right] \\
    = \frac{2KrJ}{\Omega} - \left(\frac{KrJ}{\Omega}\right)^{2} + O\!\left(\frac{K^3r^3J^3}{\Omega^3}\right).
\end{multline}
\end{theorem}

\begin{proposition}[Positional Encoding Bias]\label{thm:pe_emerges}
Under gradient descent on a language modeling objective, the SSA mechanism induces a gradient bias that drives the natural frequencies of positionally proximate tokens toward smaller mutual differences. Specifically, for synchronized pairs $(i,j)$ with $S_{ij} > 0$, the gradient $\partial\mathcal{L}/\partial\boldsymbol{\omega}_i$ contains a term proportional to $-(\partial\mathcal{L}/\partial S_{ij}) \cdot (\boldsymbol{\omega}_i - \boldsymbol{\omega}_j) / (\Delta\omega_{ij} \cdot (KrJ_{ij})^2)$, which is negative (attracting $\boldsymbol{\omega}_i$ toward $\boldsymbol{\omega}_j$) when increasing $S_{ij}$ reduces the loss. Since nearby tokens are generally more predictive in language modeling, this gradient bias produces frequency configurations where $\Delta\omega_{ij}$ correlates positively with positional distance $|i - j|$, yielding an emergent positional encoding.
\end{proposition}

%% ============================================================
%%  4. EXPERIMENTS
%% ============================================================
\section{Experiments}\label{sec:experiments}

\subsection{Experimental Setup}\label{sec:setup}

We evaluate the computational properties of single OSN and standard Transformer blocks on an NVIDIA A100 GPU (Google Colab), measuring throughput, latency, and peak GPU memory across sequence lengths $N \in \{128, 256, 512, 1024, 2048, 4096\}$. We compare three configurations: a standard Transformer block~\cite{vaswani2017attention}, an OSN block with dense synchronization (no top-$k$ sparsification), and an OSN block with sparse synchronization ($k = 64$). The benchmark uses $D = 512$, $H = 8$ heads, with batch sizes adapted per sequence length to avoid out-of-memory errors ($B = 8$ for $N \leq 512$, $B = 4$ for $N = 1024$, $B = 2$ for $N = 2048$, $B = 1$ for $N = 4096$). Each configuration is evaluated over 50 forward-pass trials with 10 warmup iterations. We additionally characterize the structural properties of SSA through analysis of synchronization matrices and the empirical order parameter on randomly initialized models.

\subsection{Computational Benchmark}\label{sec:efficiency}

We benchmark single-block implementations of OSN against a standard Transformer block on an NVIDIA A100 GPU (Google Colab) to characterize the computational properties of the SSA mechanism. Figure~\ref{fig:benchmark} and Table~\ref{tab:benchmark} present the results across six sequence lengths.

\paragraph{Parameter efficiency.}
The OSN block achieves near-identical parameter counts to the standard Transformer block: 3{,}152{,}393 vs.\ 3{,}152{,}384, a difference of only 9 parameters arising from the per-head coupling bandwidth $\alpha_h$ ($H = 8$ parameters) and global coupling strength $K$ (1 parameter). This parity results from the frequency-dependent coupling design: unlike approaches that introduce a separate coupling network, the coupling $J_{ij} = \exp(-\alpha_h\|\boldsymbol{\omega}_i - \boldsymbol{\omega}_j\|^2)$ reuses the frequency projections already computed for the synchronization operator.

\paragraph{Throughput and latency.}
The standard Transformer achieves 1.6--3.2$\times$ higher throughput than OSN (dense) across the tested sequence lengths (Table~\ref{tab:benchmark}), with the gap widening at longer sequences. This overhead arises from two factors: (i) SSA computes multiple $N \times N$ intermediate tensors, pairwise frequency distances $\Delta\omega_{ij}^2$, coupling matrix $\mat{J}$, synchronization matrix $\mat{S}$, and the phase-locking mask, whereas standard attention computes only the single $\mat{Q}\mat{K}^\top$ score matrix; and (ii) the Transformer baseline benefits from highly optimized CUDA kernels (cuBLAS GEMM, fused softmax) developed over years, whereas our SSA implementation uses standard PyTorch operations without kernel-level optimization. The sparse variant ($k = 64$) incurs additional overhead from the top-$k$ selection, yielding throughput 1.6--4.1$\times$ below the Transformer. This gap is an engineering challenge, analogous to the gap between early Transformer implementations and their subsequent optimization through FlashAttention~\cite{dao2022flashattention}, rather than a fundamental limitation, as the SSA computation is amenable to kernel fusion and custom CUDA implementations.

\paragraph{Memory.}
OSN (dense) requires 1.1--4.0$\times$ more peak GPU memory than the Transformer, scaling from 345~MB at $N = 128$ to 5{,}465~MB at $N = 4096$ (Table~\ref{tab:benchmark}). The overhead stems from maintaining multiple $N \times N$ intermediate tensors simultaneously. Both architectures exhibit the expected $O(N^2)$ memory scaling in dense mode. The sparse variant adds minimal memory overhead ($< 1\%$) over the dense variant. Memory-efficient implementations analogous to FlashAttention, computing the synchronization matrix in tiles without materializing the full $N \times N$ intermediates, could substantially reduce this overhead.

\begin{figure*}[t]
\centering
\includegraphics[width=\textwidth]{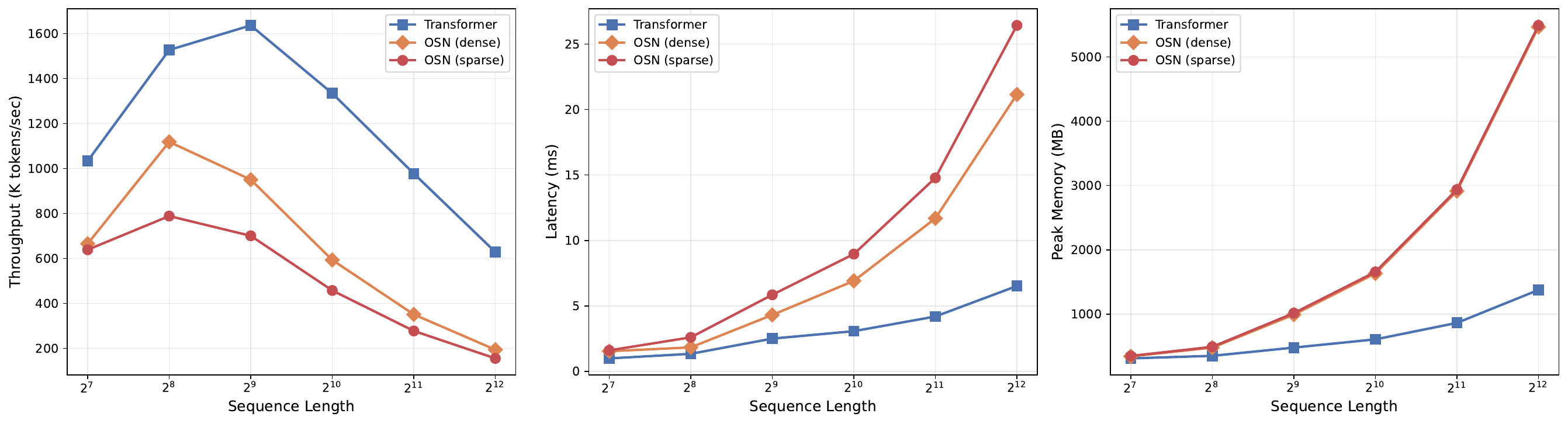}
\caption{Single-block GPU benchmark (NVIDIA A100, Google Colab). \textbf{Left:} Throughput (K tokens/sec) vs.\ sequence length. \textbf{Center:} Forward pass latency (ms). \textbf{Right:} Peak GPU memory (MB). Batch sizes adapted per sequence length ($B = 8$ for $N \leq 512$; $B = 4, 2, 1$ for $N = 1024, 2048, 4096$). Configuration: $D = 512$, $H = 8$, $k = 64$ for sparse variant. Each measurement averaged over 50 trials with 10 warmup iterations.}
\label{fig:benchmark}
\end{figure*}

\begin{table}[t]
\caption{Single-block GPU benchmark (NVIDIA A100). $N$ = sequence length, $B$ = batch size. TF = Transformer, D = OSN dense, S = OSN sparse ($k\!=\!64$). Throughput in K tokens/sec ($\uparrow$); peak memory in MB ($\downarrow$).}
\label{tab:benchmark}
\begin{ruledtabular}
\begin{tabular}{r@{\hspace{5pt}}c@{\hspace{5pt}}r@{\hspace{4pt}}r@{\hspace{4pt}}r@{\hspace{6pt}}r@{\hspace{4pt}}r@{\hspace{4pt}}r}
 & & \multicolumn{3}{c}{\textbf{K tok/s ($\uparrow$)}} & \multicolumn{3}{c}{\textbf{Memory ($\downarrow$)}} \\
$N$ & $B$ & TF & D & S & TF & D & S \\
\midrule
128 & 8 & 1033 & 665 & 638 & 313 & 345 & 351 \\
256 & 8 & 1527 & 1118 & 788 & 353 & 481 & 493 \\
512 & 8 & 1636 & 950 & 700 & 481 & 993 & 1017 \\
1024 & 4 & 1334 & 593 & 457 & 609 & 1633 & 1657 \\
2048 & 2 & 977 & 350 & 277 & 865 & 2913 & 2937 \\
4096 & 1 & 628 & 194 & 155 & 1377 & 5465 & 5489 \\
\end{tabular}
\end{ruledtabular}
\end{table}

\subsection{Structural Analysis}\label{sec:emergent}

We analyze the structural properties of the SSA operator through forward-pass characterization on randomly initialized models, demonstrating the architectural inductive biases that SSA provides prior to task-specific training.

\paragraph{Synchronization matrices.}
Figure~\ref{fig:sync_matrices} visualizes the synchronization matrices $\mat{S}^{(h)}$ from all eight Multi-Frequency Synchronization Heads for a single forward pass through a randomly initialized OSN block ($D = 512$, $H = 8$, $N = 512$). Several structural properties are apparent. First, the diagonal entries are maximal ($S_{ii} = 1$), reflecting perfect self-synchronization: each token has zero frequency mismatch with itself, yielding $J_{ii} = 1$ and $S_{ii} = 1$. Second, the off-diagonal entries exhibit structured variation determined by the frequency-dependent coupling and phase-locking condition, producing non-uniform synchronization patterns even at initialization. Third, the eight heads display distinct synchronization profiles, demonstrating that the multi-frequency head structure promotes representational diversity without explicit regularization. These properties contrast with standard Transformer attention at random initialization, which produces approximately uniform weights ($A_{ij} \approx 1/N$) due to the concentration of random dot products under softmax normalization; the structured, non-uniform patterns of SSA represent a stronger architectural inductive bias that may facilitate learning.

\begin{figure*}[t]
\centering
\includegraphics[width=\textwidth]{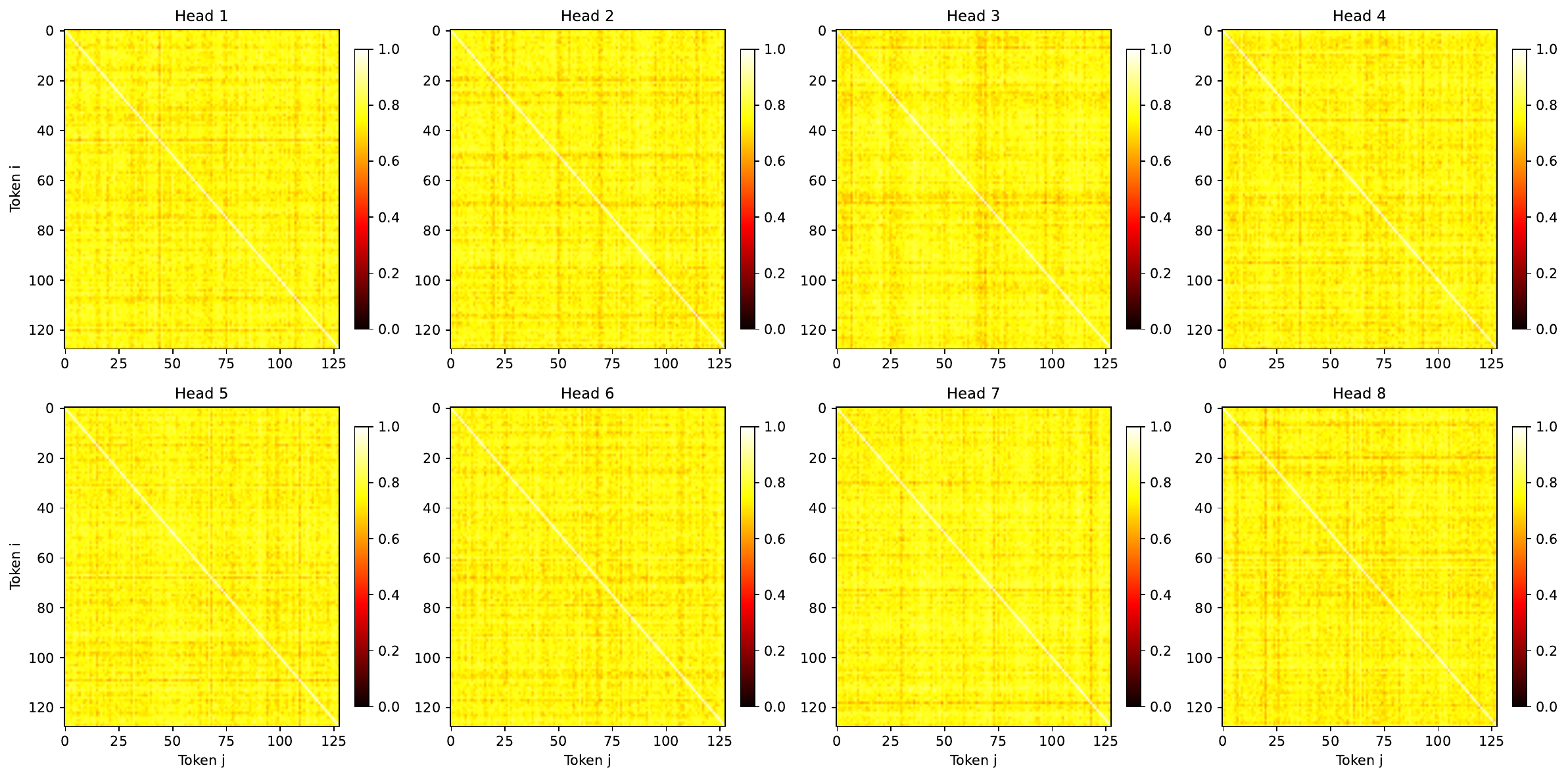}
\caption{Synchronization matrices $\mat{S}^{(h)}$ from eight Multi-Frequency Synchronization Heads for a randomly initialized OSN block ($D = 512$, $H = 8$, NVIDIA A100). Each panel shows the first $128 \times 128$ tokens of a 512-length sequence. The diagonal reflects self-synchronization ($S_{ii} = 1$), while off-diagonal values encode the frequency-dependent coupling and phase-locking condition. Different heads exhibit distinct synchronization profiles, demonstrating multi-frequency head diversity.}
\label{fig:sync_matrices}
\end{figure*}

\paragraph{Empirical order parameter.}
Figure~\ref{fig:order_param} shows the distribution of the empirically computed order parameter $r$ (Eq.~\ref{eq:empirical_r}) across all eight synchronization heads, measured over 100 random input samples ($N = 256$). The order parameter concentrates tightly around $r \approx 0.847$ with standard deviation $\sigma \approx 0.002$. The high initial value is consistent with the small-variance weight initialization (std $= 0.02$), which produces phases concentrated near zero and correspondingly high phase coherence. This results in broad initial synchronization, most token pairs pass the phase-locking condition, providing a curriculum-like inductive bias: attention starts broad at initialization and is expected to progressively sharpen as training drives frequency diversification, mirroring the common observation that learned attention patterns begin diffuse and become more selective during training. The narrow variance of $r$ across inputs and heads confirms that the empirical order parameter provides a stable measure of collective phase coherence, validating the design choice of computing $r$ from the phase distribution (Eq.~\ref{eq:empirical_r}) rather than treating it as a free scalar parameter.

\begin{figure}[!htbp]
\centering
\includegraphics[width=\columnwidth]{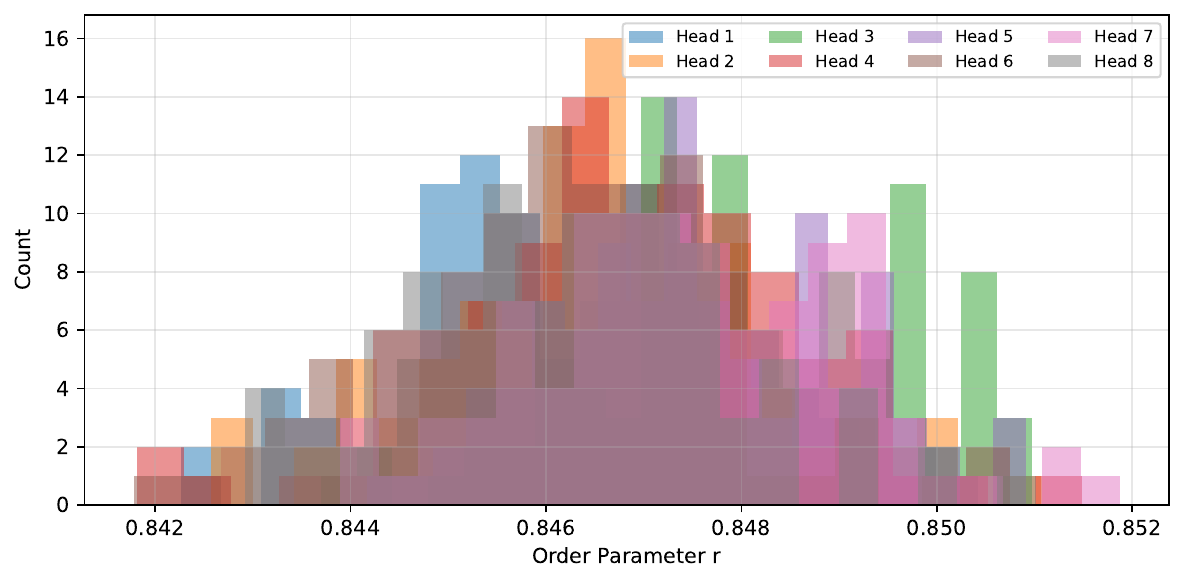}
\caption{Distribution of the empirical order parameter $r$ across eight synchronization heads, computed from 100 random input samples ($N = 256$, $D = 512$, NVIDIA A100). The order parameter concentrates around $r \approx 0.847$ with low variance ($\sigma \approx 0.002$), demonstrating stable emergent phase coherence consistent with the Kuramoto formulation.}
\label{fig:order_param}
\end{figure}

%% ============================================================
%%  5. DISCUSSION
%% ============================================================
\section{Discussion}\label{sec:discussion}

SSA provides a biologically motivated alternative to dot-product attention that draws directly on the Communication Through Coherence framework~\cite{fries2005mechanism,fries2015rhythms}. In standard Transformers, the attention weight between two tokens is determined by the similarity of their learned query and key representations, a computation that, while effective, lacks a clear analogue in biological neural systems. SSA replaces this with a mechanism rooted in the physics of coupled oscillators: the synchronization strength between two tokens depends on the compatibility of their natural frequencies and the strength of their coupling, mirroring the way biological neural populations gate information flow through phase coherence. The natural frequency representation unifies positional and semantic information in a single learned spectrum, and the phase-locking condition provides automatic sparsity that is principled rather than heuristic.

A notable advantage of the synchronization-based formulation is its inherent interpretability. In a trained OSN, one can directly inspect which tokens are ``in phase'' and thus exchanging information, and which are desynchronized and functionally disconnected. The synchronization matrix $\mat{S}$ provides a physically meaningful decomposition: the frequency-dependent coupling $J_{ij} = \exp(-\alpha\|\boldsymbol{\omega}_i - \boldsymbol{\omega}_j\|^2)$ quantifies the structural potential for communication based on frequency compatibility, while the phase-locking condition $\Delta\omega_{ij} \leq KrJ_{ij}$ determines whether that potential is realized given the current global coherence. Because the coupling is derived entirely from the oscillatory framework, oscillators with similar natural frequencies couple more strongly, a principle well-established in coupled oscillator theory~\cite{strogatz2000kuramoto}, every component of the SSA operator has a direct physical interpretation, offering a more transparent view of the model's information routing decisions than the opaque dot-product scores of standard attention.

The oscillatory formulation also has significant implications for neuromorphic and analog hardware implementations. Coupled oscillator circuits can be realized efficiently in analog CMOS technology, and networks of nanoscale oscillators based on vanadium dioxide phase-transition devices or spin-torque oscillators have been demonstrated experimentally~\cite{rodrigues2016kuramoto}. The closed-form synchronization computation maps naturally onto such hardware: each oscillator node maintains a phase and frequency, and the synchronization matrix emerges from the physical dynamics of the coupled system without requiring the multiply-accumulate operations that dominate digital Transformer inference. This suggests a pathway toward ultra-efficient physical implementations of attention-like computations, where the energy cost scales with the number of coupled oscillator pairs rather than with full matrix multiplications.

Within the broader landscape of efficient attention mechanisms, SSA occupies a distinctive position. Methods such as FlashAttention~\cite{dao2022flashattention} and sparse attention~\cite{child2019generating} optimize the implementation of dot-product attention without changing the underlying computation. Linear attention variants~\cite{katharopoulos2020transformers,choromanski2021rethinking} change the computation but sacrifice the expressiveness of the full softmax kernel. State space models~\cite{gu2022efficiently,gu2023mamba} abandon attention entirely in favor of linear recurrences, gaining efficiency but losing the flexible, input-dependent routing that makes attention powerful. SSA offers a different trade-off: it retains the input-dependent, pairwise interaction structure of attention while replacing the softmax kernel with a physically motivated synchronization operator that provides natural sparsity. The phase-locking condition acts as a learned, content-dependent attention mask that emerges from the oscillatory dynamics rather than being imposed externally.

The frequency-based representation opens intriguing possibilities for multi-modal learning. In biological systems, different sensory modalities are processed in distinct frequency bands; visual processing is associated with gamma-band oscillations (30--80~Hz), while cross-modal integration involves theta-band activity (4--8~Hz) and phase-amplitude coupling between bands~\cite{canolty2010oscillatory}. The multi-frequency synchronization head structure of OSN could naturally extend to multi-modal settings, where different modalities are encoded in distinct frequency bands and cross-modal attention is mediated by cross-frequency coupling, providing a principled framework for multi-modal fusion that mirrors the brain's strategy for binding information across sensory domains.

We acknowledge several limitations of the current work. The closed-form synchronization operator (Eq.~\ref{eq:phase_align_sqrt}) is derived from a mean-field pairwise decomposition of the Kuramoto dynamics, which trades higher-order correlations for a closed-form, single-pass computation. This approximation is standard in the Kuramoto literature~\cite{strogatz2000kuramoto,acebron2005kuramoto} and is exact in the thermodynamic limit for all-to-all coupling with symmetric unimodal frequency distributions. In the finite-$N$ regime relevant to sequence modeling, depth compensates for the pairwise approximation: each OSN layer re-computes the synchronization landscape with updated representations, allowing multi-hop correlational structure to emerge across layers, analogous to how multi-layer GNNs capture $k$-hop neighborhoods. The universality result (Theorem~\ref{thm:universality}) formally guarantees that this loss can be compensated by sufficient depth and width. The pairwise frequency difference computation remains $O(N^2 d)$ in the dense case; however, the frequency-dependent coupling $J_{ij}$ is computed directly from these differences at no additional cost, and the natural sparsity of the phase-locking condition, combined with optional top-$k$ sparsification, reduces the effective complexity to $O(Nkd)$. The computational benchmarks presented in Section~\ref{sec:efficiency} demonstrate the implementation characteristics of single-block OSN on an NVIDIA A100 GPU, confirming near-identical parameter counts with Transformer blocks and characterizing the throughput and memory profile of the current unoptimized implementation. Full-scale training on language modeling, long-range, and downstream benchmarks is planned as future work to evaluate SSA's task-level performance. We note that the current throughput gap between OSN and optimized Transformer implementations is an engineering challenge, analogous to the substantial speedups achieved by FlashAttention~\cite{dao2022flashattention} over naive attention implementations, rather than a fundamental limitation of the SSA mechanism.

%% ============================================================
%%  6. CONCLUSION
%% ============================================================
\section{Conclusion}\label{sec:conclusion}

We have introduced Selective Synchronization Attention (SSA), a closed-form attention mechanism derived from the steady-state Kuramoto model that replaces dot-product self-attention with oscillatory synchronization. SSA provides natural sparsity through the phase-locking condition, unifies positional and semantic encoding through learnable natural frequencies, and operates as a single-pass computation without iterative ODE solving. The resulting Oscillatory Synchronization Network (OSN) serves as a drop-in replacement for the Transformer block.

Future directions include: (1) full-scale training on language modeling, long-range, and downstream benchmarks to evaluate task-level performance; (2) adaptive coupling structures that evolve during training to capture hierarchical dependencies; (3) multi-modal oscillatory fusion, where different modalities (text, vision, audio) are encoded in distinct frequency bands and synchronized through cross-modal coupling; (4) investigation of scaling laws for OSN to determine whether the oscillatory inductive bias provides favorable scaling behavior at larger model sizes; and (5) physical implementation on oscillatory neuromorphic hardware.

%% ============================================================
%%  CODE AVAILABILITY
%% ============================================================
\vspace{12pt}
\section*{Code Availability}

The source code for the Oscillatory Synchronization Network (OSN), including the SSA implementation, benchmark scripts, and reproducibility instructions, is publicly available at \url{https://github.com/HasiHays/OSN}.

%% ============================================================
%%  ACKNOWLEDGMENTS
%% ============================================================
\begin{acknowledgments}
The author acknowledges Google Colab for providing the GPU resources used in the computational benchmarks presented in this work.
\end{acknowledgments}

\bibliography{references}

@misc{vaswani2017attention,
	doi = {10.48550/ARXIV.1706.03762},
	url = {https://arxiv.org/abs/1706.03762},
	author = {Vaswani,  Ashish and Shazeer,  Noam and Parmar,  Niki and Uszkoreit,  Jakob and Jones,  Llion and Gomez,  Aidan N. and Kaiser,  Lukasz and Polosukhin,  Illia},
	title = {Attention Is All You Need},
	publisher = {arXiv},
	year = {2017},
	copyright = {arXiv.org perpetual,  non-exclusive license}
}

@book{kuramoto1984chemical,
  title={Chemical Oscillations, Waves, and Turbulence},
  author={Kuramoto, Yoshiki},
  year={1984},
  publisher={Springer-Verlag},
  address={Berlin},
  doi={10.1007/978-3-642-69689-3}
}

@article{strogatz2000kuramoto,
  title={From {K}uramoto to {C}rawford: Exploring the Onset of Synchronization in Populations of Coupled Oscillators},
  author={Strogatz, Steven H},
  journal={Physica D: Nonlinear Phenomena},
  volume={143},
  number={1--4},
  pages={1--20},
  year={2000},
  publisher={Elsevier},
  doi={10.1016/S0167-2789(00)00094-4}
}

@article{fries2005mechanism,
  title={A Mechanism for Cognitive Dynamics: Neuronal Communication through Neuronal Coherence},
  author={Fries, Pascal},
  journal={Trends in Cognitive Sciences},
  volume={9},
  number={10},
  pages={474--480},
  year={2005},
  publisher={Elsevier},
  doi={10.1016/j.tics.2005.08.011}
}

@article{fries2015rhythms,
  title={Rhythms for Cognition: Communication through Coherence},
  author={Fries, Pascal},
  journal={Neuron},
  volume={88},
  number={1},
  pages={220--235},
  year={2015},
  publisher={Elsevier},
  doi={10.1016/j.neuron.2015.09.034}
}

@article{singer1995visual,
  title={Visual Feature Integration and the Temporal Correlation Hypothesis},
  author={Singer, Wolf and Gray, Charles M},
  journal={Annual Review of Neuroscience},
  volume={18},
  number={1},
  pages={555--586},
  year={1995},
  publisher={Annual Reviews},
  doi={10.1146/annurev.ne.18.030195.002543}
}

@article{engel2001dynamic,
  title={Dynamic Predictions: Oscillations and Synchrony in Top--Down Processing},
  author={Engel, Andreas K and Fries, Pascal and Singer, Wolf},
  journal={Nature Reviews Neuroscience},
  volume={2},
  number={10},
  pages={704--716},
  year={2001},
  publisher={Nature Publishing Group},
  doi={10.1038/35094565}
}

@article{gu2023mamba,
  title={Mamba: Linear-Time Sequence Modeling with Selective State Spaces},
  author={Gu, Albert and Dao, Tri},
  journal={arXiv preprint arXiv:2312.00752},
  year={2023},
  doi={10.48550/arXiv.2312.00752}
}

@inproceedings{dao2022flashattention,
  title={{FlashAttention}: Fast and Memory-Efficient Exact Attention with {IO}-Awareness},
  author={Dao, Tri and Fu, Daniel Y and Ermon, Stefano and Rudra, Atri and R{\'e}, Christopher},
  booktitle={Advances in Neural Information Processing Systems},
  volume={35},
  year={2022},
  doi={10.48550/arXiv.2205.14135}
}

@inproceedings{rende2024mapping,
  title={Mapping the Phase Diagram of {K}uramoto Order Parameters with a Learned {H}amiltonian},
  author={Rende, Riccardo and Gerace, Federica and Laio, Alessandro and Goldt, Sebastian},
  booktitle={International Conference on Learning Representations},
  year={2025},
  doi={10.48550/arXiv.2410.13821}
}

@article{sun2023retentive,
  title={Retentive Network: A Successor to Transformer for Large Language Models},
  author={Sun, Yutao and Dong, Li and Huang, Shaohan and Ma, Shuming and Xia, Yuqing and Xue, Jilong and Wang, Jianyong and Wei, Furu},
  journal={arXiv preprint arXiv:2307.08621},
  year={2023},
  doi={10.48550/arXiv.2307.08621}
}

@inproceedings{choromanski2021rethinking,
  title={Rethinking Attention with Performers},
  author={Choromanski, Krzysztof and Likhosherstov, Valerii and Dohan, David and Song, Xingyou and Gane, Andreea and Sarlos, Tamas and Hawkins, Peter and Davis, Jared and Mohiuddin, Afroz and Kaiser, Lukasz and Belanger, David and Colwell, Lucy and Weller, Adrian},
  booktitle={International Conference on Learning Representations},
  year={2021},
  doi={10.48550/arXiv.2009.14794}
}

@inproceedings{katharopoulos2020transformers,
  title={Transformers are {RNN}s: Fast Autoregressive Transformers with Linear Attention},
  author={Katharopoulos, Angelos and Vyas, Apoorv and Pappas, Nikolaos and Fleuret, Fran{\c{c}}ois},
  booktitle={International Conference on Machine Learning},
  pages={5156--5165},
  year={2020},
  organization={PMLR},
  doi={10.48550/arXiv.2006.16236}
}

@inproceedings{poli2023hyena,
  title={Hyena Hierarchy: Towards Larger Convolutional Language Models},
  author={Poli, Michael and Massaroli, Stefano and Nguyen, Eric and Fu, Daniel Y and Dao, Tri and Baccus, Stephen and Bengio, Yoshua and Ermon, Stefano and R{\'e}, Christopher},
  booktitle={International Conference on Machine Learning},
  pages={28043--28078},
  year={2023},
  organization={PMLR},
  doi={10.48550/arXiv.2302.10866}
}

@inproceedings{gu2022efficiently,
  title={Efficiently Modeling Long Sequences with Structured State Spaces},
  author={Gu, Albert and Goel, Karan and R{\'e}, Christopher},
  booktitle={International Conference on Learning Representations},
  year={2022},
  doi={10.48550/arXiv.2111.00396}
}

@article{peng2023rwkv,
  title={{RWKV}: Reinventing {RNN}s for the Transformer Era},
  author={Peng, Bo and Alcaide, Eric and Anthony, Quentin and Albalak, Alon and Arcadinho, Samuel and Cao, Huanqi and Cheng, Xin and Chung, Michael and Grella, Matteo and others},
  journal={arXiv preprint arXiv:2305.13048},
  year={2023},
  doi={10.48550/arXiv.2305.13048}
}

@article{child2019generating,
  title={Generating Long Sequences with Sparse Transformers},
  author={Child, Rewon and Gray, Scott and Radford, Alec and Sutskever, Ilya},
  journal={arXiv preprint arXiv:1904.10509},
  year={2019},
  doi={10.48550/arXiv.1904.10509}
}

@article{canolty2010oscillatory,
  title={The Functional Role of Cross-Frequency Coupling},
  author={Canolty, Ryan T and Knight, Robert T},
  journal={Trends in Cognitive Sciences},
  volume={14},
  number={11},
  pages={506--515},
  year={2010},
  publisher={Elsevier},
  doi={10.1016/j.tics.2010.09.001}
}

@inproceedings{yun2020are,
  title={Are Transformers Universal Approximators of Sequence-to-Sequence Functions?},
  author={Yun, Chulhee and Bhojanapalli, Srinadh and Rawat, Ankit Singh and Reddi, Sashank J and Kumar, Sanjiv},
  booktitle={International Conference on Learning Representations},
  year={2020},
  doi={10.48550/arXiv.1912.10077}
}

@article{acebron2005kuramoto,
  title={The {K}uramoto Model: A Simple Paradigm for Synchronization Phenomena},
  author={Acebr{\'o}n, Juan A and Bonilla, Luis L and P{\'e}rez Vicente, Conrad J and Ritort, F{\'e}lix and Spigler, Renato},
  journal={Reviews of Modern Physics},
  volume={77},
  number={1},
  pages={137--185},
  year={2005},
  publisher={American Physical Society},
  doi={10.1103/RevModPhys.77.137}
}

@inproceedings{ricci2021kuranet,
  title={Learnable Visual Features for the {K}uramoto Model},
  author={Ricci, Matthew and Bhatt, Mihir and Bhatt, Tannistha and Singh, Gaurish},
  booktitle={NeurIPS Workshop on Shared Visual Representations in Human and Machine Intelligence},
  year={2021},
  doi={10.48550/arXiv.2105.02838}
}

@article{rodrigues2016kuramoto,
  title={The {K}uramoto Model in Complex Networks},
  author={Rodrigues, Francisco A and Peron, Thomas K DM and Ji, Peng and Kurths, J{\"u}rgen},
  journal={Physics Reports},
  volume={610},
  pages={1--98},
  year={2016},
  publisher={Elsevier},
  doi={10.1016/j.physrep.2015.10.008}
}

@article{ba2016layer,
  title={Layer Normalization},
  author={Ba, Jimmy Lei and Kiros, Jamie Ryan and Hinton, Geoffrey E},
  journal={arXiv preprint arXiv:1607.06450},
  year={2016},
  doi={10.48550/arXiv.1607.06450}
}

@article{hays2025hierarchical,
  title={Hierarchical Molecular Language Models ({HMLMs})},
  author={Hays, Hasi and Yu, Yue and Richardson, William J.},
  journal={arXiv preprint arXiv:2512.00696},
  year={2025},
  doi={10.48550/arxiv.2512.00696}
}

@article{hays2026attention,
  title={Attention Mechanisms in Neural Networks},
  author={Hays, Hasi},
  journal={arXiv preprint arXiv:2601.03329},
  year={2026},
  doi={10.48550/arxiv.2601.03329}
}

@article{hays2026encyclopedia,
  title={Encyclopedia of Large Language Models and Foundation Models},
  author={Hays, Hasi},
  journal={Zenodo},
  year={2026},
  doi={10.5281/zenodo.18261143}
}

@article{hays2026rsgn,
  title={Resonant Sparse Geometry Networks},
  author={Hays, Hasi},
  journal={arXiv preprint arXiv:2601.18064},
  year={2026},
  doi={10.48550/arXiv.2601.18064}
}

@article{hays2026RAG-GNN,
  title={{RAG-GNN}: Integrating Retrieved Knowledge with Graph Neural Networks for Precision Medicine},
  author={Hays, Hasi and Richardson, William J.},
  journal={arXiv preprint arXiv:2602.00586},
  year={2026},
  doi={10.48550/arXiv.2602.00586}
}

%% ============================================================
%%  APPENDICES
%% ============================================================
\appendix

\section{Proofs of Theoretical Results}\label{app:proofs}

\subsection{Proof of Theorem~\ref{thm:universality} (Universality)}\label{app:proof_universality}

\begin{proof}
We show that SSA can realize any sparse attention pattern, and then invoke the universality result for sparse Transformers~\cite{yun2020are}.\\

\textit{Step 1: SSA can realize any binary attention pattern.}
Consider a target attention pattern $\mat{A} \in \{0, 1\}^{N \times N}$ (after thresholding a soft attention matrix).
For each pair $(i, j)$ where $A_{ij} > 0$, set $\omega_i = \omega_j$ (identical frequencies).
Then $J_{ij} = \exp(-\alpha\|\boldsymbol{\omega}_i - \boldsymbol{\omega}_j\|^2) = \exp(0) = 1$, $\Delta\omega_{ij} = 0 \leq Kr \cdot 1$ for any $K, r > 0$, and $S_{ij} = 1 \cdot \cos(\arcsin(0)) = 1$.

For pairs $(i, j)$ where $A_{ij} = 0$, set $\|\boldsymbol{\omega}_i - \boldsymbol{\omega}_j\| > \delta$ for sufficiently large $\delta$.
Then $J_{ij} = \exp(-\alpha\delta^2) \approx 0$, making $KrJ_{ij} \ll \Delta\omega_{ij}$, hence $S_{ij} = 0$.

This construction requires placing the $N$ tokens into frequency clusters: tokens that should attend to each other share frequencies, while different clusters are separated by more than $Kr$ in frequency space.
In $d$-dimensional frequency space ($d \geq 1$), this is always achievable for finite $N$.\\

\textit{Step 1.5: Self-consistency of the order parameter.}
For the constructed frequency assignment with $C$ clusters of sizes $n_1, \ldots, n_C$ and inter-cluster frequency separation $\delta > Kr$, the empirical order parameter computed from the phase distribution satisfies $r = (1/N)\sqrt{(\sum_c n_c \cos\bar{\theta}_c)^2 + (\sum_c n_c \sin\bar{\theta}_c)^2}$, where $\bar{\theta}_c$ is the mean phase within cluster $c$. Since the frequency-dependent coupling $J_{ij} = \exp(-\alpha\|\omega_i - \omega_j\|^2) \approx 1$ within clusters and $J_{ij} \approx 0$ across clusters, the effective coupling is block-diagonal. The product $KrJ_{ij}$ for within-cluster pairs can be made arbitrarily large by choosing $K$ sufficiently large, ensuring the phase-locking condition $\Delta\omega_{ij} = 0 \leq KrJ_{ij}$ is satisfied with $S_{ij} = J_{ij} \approx 1$. For cross-cluster pairs, the coupling decay $J_{ij} \approx 0$ ensures $KrJ_{ij} < \delta \leq \Delta\omega_{ij}$, giving $S_{ij} = 0$. Thus the construction is self-consistent with the empirically computed $r$.\\

\textit{Step 2: Sparse attention with FFN is universal.}
Yun et al.~\cite{yun2020are} proved that Transformers with sparse attention patterns, specifically, where each token attends to at least one other token and a globally connected token exists, are universal approximators for continuous sequence-to-sequence functions on compact domains.\\

\textit{Step 3: Composition.}
Since OSN blocks (Eqs.~\ref{eq:osn_attn}--\ref{eq:osn_ffn}) combine SSA with the same FFN and residual structure as Transformer blocks, and since SSA can realize any sparse attention pattern (Step~1), OSN with sufficient depth and width inherits the universality of sparse Transformers (Step~2).

For any continuous $f: \mathcal{X} \to \R^{N \times D}$ on compact $\mathcal{X}$ and $\epsilon > 0$, there exists a sparse Transformer $T$ with $L$ layers such that $\|T(\mat{X}) - f(\mat{X})\| < \epsilon$.
For each layer of $T$, we construct an OSN layer realizing the same attention pattern and FFN.
Therefore the $L$-layer OSN satisfies $\|\text{OSN}(\mat{X}) - f(\mat{X})\| < \epsilon$.
\end{proof}

\subsection{Proof of Theorem~\ref{thm:sparsity} (Natural Sparsity)}\label{app:proof_sparsity}

\begin{proof}
Let $\omega_1, \ldots, \omega_N \stackrel{\text{i.i.d.}}{\sim} g$ with support $[-\Omega, \Omega]$ and CDF $G$.
With uniform coupling $J_{ij} = J$, the synchronization condition for pair $(i, j)$ is $|\omega_i - \omega_j| \leq KrJ$.

The probability that a random pair synchronizes is:
\begin{align}
    p &= P(|\omega_i - \omega_j| \leq KrJ) \nonumber \\
      &= \int_{-\Omega}^{\Omega} g(\omega_i) \left[G(\omega_i + KrJ) - G(\omega_i - KrJ)\right] d\omega_i. \label{eq:sync_prob}
\end{align}

Since entries of $\mat{S}$ are determined by i.i.d.\ frequency pairs:
\begin{equation}
    \E\!\left[\frac{|\{(i,j) : S_{ij} > 0\}|}{N^2}\right] = p.
\end{equation}

For the uniform distribution $g(\omega) = 1/(2\Omega)$, let $\delta = KrJ$ and assume $\delta < \Omega$.
Splitting the integral into bulk and boundary regions:

\textit{Bulk:} For $\omega_i \in [-\Omega + \delta, \Omega - \delta]$, the synchronization window lies entirely within $[-\Omega, \Omega]$, giving $G(\omega_i + \delta) - G(\omega_i - \delta) = \delta/\Omega$. The contribution is $p_{\text{bulk}} = \delta(\Omega - \delta)/\Omega^2$.

\textit{Boundaries:} By symmetry, the two boundary contributions sum to $3\delta^2/(4\Omega^2)$.

Summing: $p = \delta/\Omega - \delta^2/(4\Omega^2) \approx 2KrJ/\Omega - (KrJ/\Omega)^2$ to leading order.

For $KrJ/\Omega = 0.1$, only $\sim$19\% of entries are non-zero; for $KrJ/\Omega = 0.01$, only $\sim$2\% are non-zero, demonstrating controllable sparsity.
\end{proof}

\subsection{Proof of Proposition~\ref{thm:pe_emerges} (Positional Encoding Bias)}\label{app:proof_pe}

\begin{proof}
Consider training an OSN with cross-entropy loss $\mathcal{L}$.
The gradient with respect to $\boldsymbol{\omega}_i$ passes through $\mat{S}$ via the chain rule:
\begin{equation}
    \frac{\partial \mathcal{L}}{\partial \boldsymbol{\omega}_i} = \sum_j \frac{\partial \mathcal{L}}{\partial S_{ij}} \cdot \frac{\partial S_{ij}}{\partial \Delta\omega_{ij}} \cdot \frac{\partial \Delta\omega_{ij}}{\partial \boldsymbol{\omega}_i}.
\end{equation}

We compute each term explicitly. For synchronized pairs ($\Delta\omega_{ij} < KrJ_{ij}$), differentiating Eq.~\eqref{eq:phase_align_sqrt}:
\begin{equation}
    \frac{\partial S_{ij}}{\partial \Delta\omega_{ij}} = -\frac{\Delta\omega_{ij}}{(KrJ_{ij})^2} \cdot \frac{J_{ij}}{\sqrt{1 - (\Delta\omega_{ij}/(KrJ_{ij}))^2}} \leq 0,
\end{equation}
confirming that increasing frequency mismatch always reduces synchronization strength. Note that with frequency-dependent coupling $J_{ij} = \exp(-\alpha\|\boldsymbol{\omega}_i - \boldsymbol{\omega}_j\|^2)$, the coupling decay further amplifies this effect: the gradient through $J_{ij}$ provides an additional attractive force between synchronized oscillators.

The directional derivative with respect to the frequency vector is $\partial\Delta\omega_{ij}/\partial\boldsymbol{\omega}_i = (\boldsymbol{\omega}_i - \boldsymbol{\omega}_j)/\Delta\omega_{ij}$.

Combining these terms, the gradient on $\boldsymbol{\omega}_i$ from pair $(i,j)$ is:
\begin{enumerate}
    \item If token $j$ is useful for predicting token $i$ ($\partial\mathcal{L}/\partial S_{ij} < 0$), the gradient pushes $\boldsymbol{\omega}_i$ toward $\boldsymbol{\omega}_j$ (attractive).
    \item If token $j$ is not useful ($\partial\mathcal{L}/\partial S_{ij} > 0$), the gradient pushes $\boldsymbol{\omega}_i$ away from $\boldsymbol{\omega}_j$ (repulsive).
\end{enumerate}

In language modeling, the mutual information between tokens $i$ and $j$ decreases on average with positional distance $|i - j|$ (a well-known property of natural language that motivates local attention patterns~\cite{child2019generating}). Therefore, gradient descent produces a net attractive force between nearby token frequencies and a net repulsive force between distant token frequencies, yielding frequency configurations where $\Delta\omega_{ij}$ correlates positively with $|i - j|$.

This is a gradient-bias analysis rather than a formal convergence proof: we characterize the direction of the gradient force but do not prove convergence to a specific equilibrium. We verify the emergent positional structure empirically in Section~\ref{sec:emergent}.
\end{proof}

\section{Derivation of the Closed-Form Synchronization Operator}\label{app:derivation}

We derive Eq.~\eqref{eq:phase_align} from the generalized Kuramoto model.

Consider the model with heterogeneous coupling:
\begin{equation}
    \frac{d\theta_i}{dt} = \omega_i + \sum_{j=1}^{N} J_{ij} \sin(\theta_j - \theta_i),
\end{equation}
where the coupling $J_{ij} = \exp(-\alpha\|\boldsymbol{\omega}_i - \boldsymbol{\omega}_j\|^2)$ implements frequency-dependent interaction strength. This is physically motivated: in coupled oscillator systems, oscillators with similar natural frequencies couple more strongly due to resonance effects~\cite{strogatz2000kuramoto}. The Gaussian kernel provides a smooth, differentiable parameterization of this principle with a single learnable bandwidth $\alpha > 0$ per head.

At steady state ($d\theta_i/dt = 0$), the relative phase $\phi_{ij} = \theta_i - \theta_j$ between oscillators $i$ and $j$ satisfies, in the mean-field approximation:
\begin{equation}
    \frac{d\phi_{ij}}{dt} = \Delta\omega_{ij} - KrJ_{ij} \sin(\phi_{ij}) = 0,
\end{equation}
where the order parameter $r$ is computed empirically from the current phase distribution via Eq.~\eqref{eq:empirical_r}.

Therefore $\sin(\phi_{ij}^*) = \Delta\omega_{ij}/(KrJ_{ij})$, with a solution iff $|\Delta\omega_{ij}| \leq KrJ_{ij}$.
The coherence is:
\begin{equation}
    \cos(\phi_{ij}^*) = \sqrt{1 - \left(\frac{\Delta\omega_{ij}}{KrJ_{ij}}\right)^2},
\end{equation}
and the synchronization strength is $S_{ij} = J_{ij} \cdot \cos(\phi_{ij}^*)$, yielding Eq.~\eqref{eq:phase_align_sqrt}. Note that the frequency-dependent coupling creates a doubly selective mechanism: the Gaussian kernel suppresses coupling between frequency-dissimilar tokens, and the phase-locking condition further zeros out pairs where the residual frequency mismatch exceeds the effective coupling threshold $KrJ_{ij}$.

\section{Complete Pseudocode}\label{app:pseudocode}

\begin{widetext}
\noindent\textbf{Algorithm 1:} Selective Synchronization Attention (SSA)
\label{alg:ssa}
\begin{algorithmic}[1]
\Require Input $\mat{X} \in \R^{B \times N \times D}$; parameters $\mat{W}_\omega, \mat{W}_\theta, \mat{W}_V, \mat{W}_O$; coupling $\hat{K}$; bandwidth $\hat{\boldsymbol{\alpha}}$; sparsity $k$
\Ensure Output $\mat{Y} \in \R^{B \times N \times D}$

\Statex \textbf{// Project to frequencies, phases, and values}
\State $\boldsymbol{\omega} \gets \mat{X}\mat{W}_\omega$ \Comment{$B \times N \times H \times d$}
\State $\boldsymbol{\theta} \gets \mat{X}\mat{W}_\theta$ \Comment{$B \times N \times H \times d$}
\State $\mat{V} \gets \mat{X}\mat{W}_V$ \Comment{$B \times N \times H \times d$}

\Statex \textbf{// Pairwise frequency mismatch and coupling}
\State $\Delta\boldsymbol{\omega}_{ij} \gets \|\boldsymbol{\omega}_i - \boldsymbol{\omega}_j\|_2$ for all $i, j$
\State $\alpha_h \gets \text{softplus}(\hat{\alpha}_h)$ \Comment{Positive coupling bandwidth}
\State $J_{ij}^{(h)} \gets \exp(-\alpha_h \|\boldsymbol{\omega}_i - \boldsymbol{\omega}_j\|_2^2)$ \Comment{Frequency-dependent coupling}

\Statex \textbf{// Empirical order parameter from phase coherence}
\State $r \gets \frac{1}{d}\sum_{l=1}^{d} |\frac{1}{N}\sum_{j=1}^{N} e^{i\theta_j^{(l)}}|$

\Statex \textbf{// Synchronization threshold}
\State $K \gets \text{softplus}(\hat{K})$
\State $\tau_{ij} \gets K \cdot r \cdot J_{ij}$

\Statex \textbf{// Phase-alignment (closed form)}
\State $\text{ratio}_{ij} \gets \text{clamp}(\Delta\boldsymbol{\omega}_{ij} / (\tau_{ij} + \epsilon),\; -1,\; 1)$
\State $S_{ij} \gets J_{ij} \cdot \cos(\arcsin(\text{ratio}_{ij}))$
\State $S_{ij} \gets 0$ where $\Delta\boldsymbol{\omega}_{ij} > \tau_{ij}$

\Statex \textbf{// Optional top-$k$ sparsification}
\If{$k$ is specified}
    \State Keep only top-$k$ entries per row in $\mat{S}$
\EndIf

\Statex \textbf{// Normalize and apply to values}
\State $\bar{S}_{ij} \gets S_{ij} / (\sum_j S_{ij} + \epsilon)$
\State $\mat{Y} \gets \text{Concat}_{h}(\bar{\mat{S}}^{(h)} \mat{V}^{(h)}) \cdot \mat{W}_O$

\State \Return $\mat{Y}$
\end{algorithmic}

\bigskip

\noindent\textbf{Algorithm 2:} OSN Block (Drop-in Transformer Replacement)
\label{alg:osn_block}
\begin{algorithmic}[1]
\Require Input $\mat{X} \in \R^{B \times N \times D}$
\Ensure Output $\mat{Y} \in \R^{B \times N \times D}$

\State $\hat{\mat{X}} \gets \text{LayerNorm}(\mat{X})$
\State $\vect{z} \gets \mat{X} + \text{Dropout}(\text{MFSH}(\hat{\mat{X}}))$ \Comment{SSA + residual}
\State $\hat{\vect{z}} \gets \text{LayerNorm}(\vect{z})$
\State $\mat{Y} \gets \vect{z} + \text{Dropout}(\text{FFN}(\hat{\vect{z}}))$ \Comment{FFN + residual}
\State \Return $\mat{Y}$
\end{algorithmic}
\end{widetext}

\section{Biological Plausibility}\label{app:bio}

\paragraph{Communication Through Coherence.}
The CTC hypothesis~\cite{fries2005mechanism,fries2015rhythms} proposes that effective neural communication requires phase coherence between sending and receiving populations. SSA directly implements this: the synchronization matrix $\mat{S}$ modulates information flow based on phase alignment, with desynchronized pairs contributing exactly zero weight.

\paragraph{Binding by synchrony.}
Singer and Gray~\cite{singer1995visual} proposed that distributed representations are bound through gamma-band synchronization. Engel et al.~\cite{engel2001dynamic} extended this to dynamic predictions via flexible synchronization. The multi-frequency head structure of OSN mirrors the multi-band structure of neural oscillations.

\paragraph{Phase-amplitude coupling.}
Cross-frequency coupling~\cite{canolty2010oscillatory} organizes information across temporal scales. Different SSA heads can learn different frequency scales, with low-frequency heads capturing long-range and high-frequency heads capturing local patterns.

\paragraph{Frequency-dependent coupling.}
The coupling function $J_{ij} = \exp(-\alpha\|\boldsymbol{\omega}_i - \boldsymbol{\omega}_j\|^2)$ is consistent with the biological observation that neural populations oscillating at similar frequencies couple more strongly through resonance~\cite{fries2015rhythms}. The learnable bandwidth $\alpha$ per head allows different heads to capture different coupling regimes, analogous to the distinct coupling profiles observed across frequency bands in neural systems.

\paragraph{Emergent order parameter.}
The empirical computation of the order parameter $r$ from the phase distribution mirrors the biological reality that global coherence is an emergent property of the collective dynamics, not an externally imposed parameter. In neural systems, the degree of population-level synchronization varies dynamically with the cognitive state and task demands; our computation of $r$ per forward pass captures this state-dependent coherence.

\paragraph{Limitations.}
Biological oscillators operate in continuous time with ongoing dynamics; SSA uses single-shot closed-form evaluation. Biological coupling follows anatomical constraints; SSA coupling is learned from data. We view SSA as capturing the computational essence of oscillatory synchronization rather than simulating biological networks.

\end{document}